\documentclass[a4paper]{article}

\usepackage{amsfonts,amsopn,amsthm}
\usepackage{geometry,authblk}
\usepackage{graphicx}
\usepackage{algorithm,algpseudocode}
\usepackage{mathtools}
\usepackage{nicematrix,booktabs,tabularx,multirow,multicol}
\usepackage{cleveref,enumitem}
\usepackage{subcaption}
\usepackage{siunitx}
\usepackage[normalem]{ulem}

\newtheorem{theorem}{Theorem}[section]

\theoremstyle{definition}
\newtheorem{definition}[theorem]{Definition}

\newtheorem{proposition}[theorem]{Proposition}

\theoremstyle{remark}

\numberwithin{equation}{section}

\newcommand{\stkout}[1]{\ifmmode\text{\sout{\ensuremath{#1}}}\else\sout{#1}\fi}

\numberwithin{equation}{section}
\crefname{equation}{}{}
\crefname{lemma}{Lemma}{Lemmas}
\crefname{proposition}{Proposition}{Proposition}
\crefname{corollary}{Corollary}{Corollaries}
\crefname{theorem}{Theorem}{Theorems}
\crefname{definition}{Definition}{Definitions}
\crefname{table}{Table}{Tables}
\crefname{figure}{Figure}{Figures}
\crefname{algorithm}{Algorithm}{Algorithms}
\crefname{fact}{Fact}{Facts}
\crefname{section}{Section}{Sections}
\crefname{subsection}{Section}{Sections}
\Crefname{subsection}{Section}{Sections}
\crefname{chapter}{Chapter}{Chapters}
\Crefname{chapter}{Chapter}{Chapters}
\crefname{appendix}{Appendix}{Appendices}
\Crefname{appendix}{Appendix}{Appendices}
\graphicspath{{pics/}}
  
\begin{document}

\title{Hybrid Least Squares\slash{}Gradient Descent Methods for MIONets}

\author[1]{Jun Choi}
\author[1]{Chang-Ock Lee}
\author[2]{Minam Moon}
\affil[1]{\small Department of Mathematical Sciences, KAIST, Daejeon 34141, KOREA}
\affil[2]{\small Department of Mathematics, Korea Military Academy, Seoul 01805, KOREA}
\date{}

\maketitle
\begin{abstract}
    In this paper, we propose an efficient hybrid least squares/gradient descent~(LSGD) method 
    for MIONets to accelerate training. This method generalizes the LSGD method for DeepONets. 
    Since MIONet is the sum of the entrywise product of multiple branch networks and a trunk 
    network, it can be viewed as a multilinear function with respect to the last layer 
    parameters of each branch network. 
    These sets of parameters can be optimized using the alternating least squares method, 
    where we solve the LS system for a single branch network in turn. 
    To handle the large-sized system matrix, we introduce Kronecker and Khatri-Rao 
    products and tensor permutation matrices to factor the 
    large matrix into small ones. 
    Our method is compatible with a general type of $L^2$ loss with regularization terms for 
    the last layer parameters of each branch, where linear operators can be applied to the 
    MIONet output in each loss term. 
\end{abstract}



\let\thefootnote\relax\footnotetext{
    \textbf{Funding:} This work was supported by Basic Science Research Program through the National Research 
    Foundation~(NRF) of Korea funded by the Ministry of Education [RS2025--25397599].
}

\section{Introduction}\label{sec1}

Thanks to the recent advances in scientific machine learning, 
the core architectures, including deep learning~(DL) and deep neural networks~(DNNs) 
have migrated to the field of scientific computing to enhance existing numerical methods 
for solving various partial differential equations~(PDEs). 
In particular, physics-informed neural network~(PINN)~\cite{Raissi2019} is the most successful 
and widely used method, where PINN represents the solution of PDE as a DNN and finds the 
solution by training the DNN using a physics-informed loss~(PI-loss) with the automatic 
differentiation method~\cite{Baydin2018}. 
However, since PINN requires separate training for different PDE instances, 
the need for a mapping between components and solutions of PDEs using DL architecture 
has emerged, which has now been generalized to neural operator mapping between function 
spaces. 
There are various examples of neural operators, including Deep Operator 
Network~(DeepONet)~\cite{Lu2021}, Fourier Neural Operator~\cite{Li2020}, 
Graph Kernel Network~\cite{Li2020neural}, PCA-based Model Reduction~\cite{Bhattacharya2021}, 
and Multi-Wavelet Neural Operator~\cite{Gupta2021}. 

Among these neural operators, DeepONet is the most widely used framework for 
neural operators, which possesses the universal approximation property. 
It consists of the inner product of outputs from two neural networks, branch and trunk, 
where the branch network encodes input functions and the trunk network encodes coordinates 
of the output function domain. 
Based on the DeepONet architecture, many variants have been proposed, such as 
POD-DeepONet~\cite{Lu2022}, 
Multifidelity DeepONet~\cite{Lu2022multifidelity}, 
NOMAD~\cite{Seidman2022}, 
Multiple-Input Operator Network~(MIONet)~\cite{Jin2022}, 
Shift-DeepONet~\cite{Hadorn2022}, 
HyperDeepONet~\cite{Lee2023hyper}, and 
Geom-DeepONet~\cite{He2024}. 

In this paper, we focus on MIONet since it is a direct generalization of DeepONet, which maps 
several input functions to a single output function with the corresponding universal 
approximation theorem (UAT)~\cite[Theorem 3.1]{Jin2022}. 
Instead of a single branch network in DeepONet, MIONet uses multiple branch networks to encode 
each input function and computes the entrywise product of the outputs from each branch to 
perform an inner product with the output of the trunk network. 

However, the training for MIONet is challenging because the entrywise product and inner 
product among several networks make the structure more complex, and a sufficiently large 
dataset is needed for meaningful training. 
This makes the conventional MIONet training with the Adam optimizer~\cite{Kingma2017} 
require very high computational resources and time. 

To optimize and accelerate MIONet training, we take a deeper look into the hybrid least 
squares/gradient descent~(LSGD) method for DeepONet~\cite{Choi2025,Cyr2020}. 
We generalize the LSGD method into the MIONet framework with vanilla structure, where 
each output layer of the branch network is a fully connected linear layer. 
For a general type of $L^2$ loss with regularization terms for the last layer parameters of 
the branch networks, we formulate a minimization problem of sums of squared 
multilinear functions in terms of the last layer parameters. 
To optimize this problem, we first fix the last layer parameters of all branches except one, 
then the minimization problem becomes an LS system for the unfixed last layer parameters. 
Now, by generating and solving LS problems for the unfixed last layer parameters, 
alternating branches in sequence, we can optimize the last layer parameters of all the 
branch networks. We call this the alternating least squares~(ALS) method. 
Although each LS system is very large to handle, we can factor the large system matrix into 
smaller matrices from each branch and trunk network, where the column-wise Kronecker 
product~(Khatri-Rao product~\cite{Khatri1968}) and the usual Kronecker product are used. 
Additionally, we introduce tensor permutation matrices~\cite{Rakotonirina2005} to match the row 
order shuffled due to the Kronecker products to the lexicographic order of the dimension 
axes. 
After that, the LS system is transformed into a special type of matrix equation of the form 
$AXB+\lambda X = E$, where the coefficient matrices $A$ and $B$ are from the component matrices 
of the LS system. 
Furthermore, we provide a theorem that helps reduce the complexity of the computation of 
data tensors, when the given data tensors depend only on one input function argument 
of MIONet.
Finally, we propose the ALS plus Adam~(ALS+Adam) method as a practical algorithm for LSGD for 
MIONet, which is a modification of the LS+Adam method~\cite{Choi2025}. 

This paper is organized as follows. In \cref{sec2}, 
we introduce MIONet and provide the UAT for vanilla MIONets. 
We also briefly summarize the LSGD method~\cite{Cyr2020} and LSGD methods for 
DeepONets~\cite{Choi2025}. 
In \cref{sec3}, we formulate the minimization problem from the general type of squared $L^2$ 
loss in terms of the last layer parameters of branch networks. 
After that, we will explain how this problem can be understood as an LS system of the last 
layer parameters and present the corresponding LSGD method for MIONets. 
In \cref{sec4}, we conduct experiments on supervised learning for a nonlinear PDE 
and unsupervised learning for linear PDEs 
to compare the training performance between conventional MIONet training with Adam and 
MIONet training with ALS+Adam. 

\section{Preliminaries}\label{sec2}

In this section, we introduce MIONet with its universal approximation 
property~\cite{Jin2022} and the hybrid LSGD method for DeepONets~\cite{Choi2025}. 
Refer to~\cref{Tab:Notation} for the meaning of symbols and variables 
used in this paper. 

\begin{table}
    \caption{\textit{Notation Table.}}
    \centering
    {\footnotesize
    \begin{tabular}{c c c} 
        \toprule
        {Notation} & {Space} & {Description}\\\midrule
        {$N$} & {$\mathbb{N}$} & {Number of branch networks}\\
        {$I$} & {$\mathbb{N}$} & {Number of output nodes of branch and trunk networks}\\
        {$J_m$} & {$\mathbb{N}$} & {Number of input nodes of output layer of $m$-th branch}\\
        {$M_m$} & {$\mathbb{N}$} & {Number of discretization points for $m$-th branch input}\\
        {$d_0$} & {$\mathbb{N}$} & {Number of dimension for output function coordinate}\\\midrule
        {$\mathbf{u}^{(m)}$} & {$\mathbb{R}^{M_m}$} & {Discretized input function for $m$-th branch}\\
        {$y$} & {$\mathbb{R}^{d_0}$} & {Coordinate for output function}\\
        {$\mathbf{b}_{m}(\mathbf{u}^{(m)})$} & {$\mathbb{R}^I$} & 
        {Output of $m$-th branch network of vanilla MIONet}\\
        {$\mathbf{t}(y)$} & {$\mathbb{R}^I$} & 
        {Output of trunk network of vanilla MIONet}\\
        {$\tilde{\mathbf{b}}_{m}(\mathbf{u}^{(m)})$} & {$\mathbb{R}^{J_m}$} & 
        {Output of layer before last layer of $m$-th branch}\\
        {$C_m$} & {$\mathbb{R}^{I\times J_m}$} & {Last layer parameter matrix of $m$-th branch}\\\midrule
        {$\theta^{B}_m$} & {-} & {Hidden layer parameters of $m$-th branch}\\
        {$\theta^T$} & {-} & {Parameters of trunk network}\\
        {$\theta^{L}_m$} & {$\mathbb{R}^{I J_m}$} & 
        {Last layer parameters of $m$-th branch, $\theta^{L} = \text{vec}(C_{m}^T)$}\\\midrule
        {{${[d_1,\dots,d_N]}_{D_1,\dots,D_N}$}} & 
        {{$\mathbb{N}$}} & 
        {\begin{tabular}{c}Big-endian order for the entries of \\
            the rank $N$ tensor in \cref{eq:LexIndDef}
        \end{tabular}} 
        \\ \midrule
        {$K$} & {$\mathbb{N}$} & {Number of loss terms except regularization terms}\\
        {$\epsilon_k$} & {$\mathbb{R}^{>0}$} & {Weight for $k$-th loss term}\\
        {$\lambda_m$} & {$\mathbb{R}^{>0}$} & {Weight for $L^2$ regularization term of $\theta^{L}_m$}\\
        {$\mathcal{L}_k$} & {$\mathcal{L}(C(\mathbb{R}),C(\mathbb{R}))$} & 
        {Linear operator for $k$-th loss term}\\
        {$D_{k}$} & {$\mathbb{N}$} & {Number of data pairs for $k$-th loss term}\\
        {$\chi_k$} & {${\left((\prod_{m=1}^{N} \mathbb{R}^{M_m}) \times \mathbb{R}^{d_0}\right)}^{D_{k}}$} & 
        {Data $(\mathbf{u}^{(1)},\dots,\mathbf{u}^{(N)},y)$ collection for $k$-th loss term}\\
        {$\mathcal{A}_k$} & {$\mathbb{R}^{D_{k}}$} & 
        {Long vector whose $d_k$-th entry is \cref{eq:EntryMION}}\\
        {$A_{k,m}$} & {$\mathbb{R}^{D_{k} \times I J_m}$} & 
        {System matrix for $k$-th loss term of LS problem in $\theta^{L}_m$}\\
        {$f_k$} & {$\mathbb{R}^{D_{k}}$} & 
        {Data for $k$-th loss term of LS problem}\\\midrule
        {$P_m$} & {$\mathbb{N}$} & {Number of input functions for $m$-th branch}\\
        {$Q_k$} & {$\mathbb{N}$} & {Number of coordinate points of $k$-th loss term}\\
        {$\beta_m$} & {${(\mathbb{R}^{M_m})}^{P_m}$} & {Set of discretized input functions for $m$-th branch}\\
        {$\tau_k$} & {${(\mathbb{R}^{d_0})}^{Q_k}$} & 
        {Set of coordinate points of $k$-th loss term}\\\midrule
        {$K_{D,\sigma}$} & {$\mathbb{R}^{{D_1} \cdots {D_m} \times {D_1} \cdots {D_m}}$} & 
        {Tensor permutation matrix in \cref{def:TensorPermMat}}\\
        {$B_m$} & {$\mathbb{R}^{P_m \times J_m}$} & {$m$-th branch pre-output matrix, 
        $\left(\tilde{b}_j (\mathbf{u}^{(m)}_p)\right)$}\\
        {$T_k$} & {$\mathbb{R}^{Q_k \times I}$} & 
        {$k$-th trunk output matrix with $\mathcal{L}_k$, 
        $\left(\mathcal{L}_k[t_i](y_{q})\right)$}\\
        {$F_k$} & {$\mathbb{R}^{P_1 \times \cdots \times P_N \times Q_k}$} & 
        {Data for $k$-th loss term in tensor form}\\
        \bottomrule 
    \end{tabular}~\label{Tab:Notation}}
\end{table}

\subsection{Multiple-Input Operator Network~(MIONet)}\label{sec21}

Jin et al.~\cite{Jin2022} proved an UAT for multiple-input 
operators on the product of Banach spaces of functions with Schauder bases. 
We refer to~\cite{Fabian2011,Hu2025,Semadeni2006} for more details of the 
Schauder basis and its canonical projections. 
The theorem states that a continuous multiple-input operator $G$ on the 
product of compact subsets of Banach spaces $X_m$ with Schauder bases can be approximated 
by the form
\begin{equation}~\label{eq:UATMION} 
    \Big\langle 
    \underbrace{(\mathbf{b}_1\circ \phi^1_{M_1})}_{\textrm{branch $1$}}
    \odot \cdots \odot 
    \underbrace{(\mathbf{b}_N \circ \phi^N_{M_N})}_{\textrm{branch $N$}}, 
    \underbrace{\mathbf{t}}_{\textrm{trunk}}\Big\rangle
\end{equation}
with continuous vector functions 
$\mathbf{b}_m\in C(\mathbb{R}^{M_m},\mathbb{R}^{I})$ and $\mathbf{t}\in Y^{I}$ 
for sufficiently large positive integers $I$ and $M_m$, 
where $\phi^m_{M_m}\colon X_m\to\mathbb{R}^{M_m}$ extracts the first $M_m$ coefficients of the 
Schauder basis representation, $Y$ is the target Banach space, and $\odot$ denotes the Hadamard 
(entrywise) product.

Here, from the inner product structure~\eqref{eq:UATMION} of vector functions $\mathbf{b}_m$ 
and $\mathbf{t}$, a multiple-input operator network (MIONet) can be constructed by replacing 
those functions with neural networks. 
In accordance with the UAT for DeepONet~\cite{Lu2021}, 
we want to use discretized function values of $u^{(m)}$ as an input for $\mathbf{b}_m$, 
where $\phi_n^m$ extracts function values at certain points. 
Note that for the sequence of distinct points ${\{t_i\}}_{i=1}^{\infty}$ in $[0,1]$ with 
$t_1=0$ and $t_2=1$, which is dense in $[0,1]$, there exists a Schauder basis 
${\{e_i\}}_{i=1}^{\infty}$ of $C({[0,1]})$ (called the Faber-Schauder basis) 
where $e_1(t) = 1$ and $e_{n}$ is chosen as a piecewise linear function with 
$e_{n}(t_{n})=1$, where the set $\{e_1,\dots,e_{n}\}$ 
forms a basis for the space of all piecewise linear functions with node points 
${\{t_i\}}_{i=1}^{n}$. 
We refer to~\cite{Fabian2011,Semadeni2006} for detailed Faber-Schauder basis 
construction of $C({[0,1]})$ and~\cite{Glenn2019} for the multivariate extension to 
$C({[0,1]}^d)$. 
Furthermore, without loss of generality, MIONet can have a vanilla structure 
similar to DeepONets in~\cite{Choi2025,Son2025}, where 
the last layer of each branch network is a fully connected layer without 
bias and activation function. 
This can be achieved by adding an identity layer to the output of each branch network. 
See \cref{fig:21} for the structure of vanilla MIONet. 
Note that MIONet with $N=1$ becomes DeepONet with the same structure and the universal 
approximation property discussed in~\cite{Choi2025}. 

\begin{figure}[tb] 
    \centering
    \includegraphics[width=0.65\linewidth]{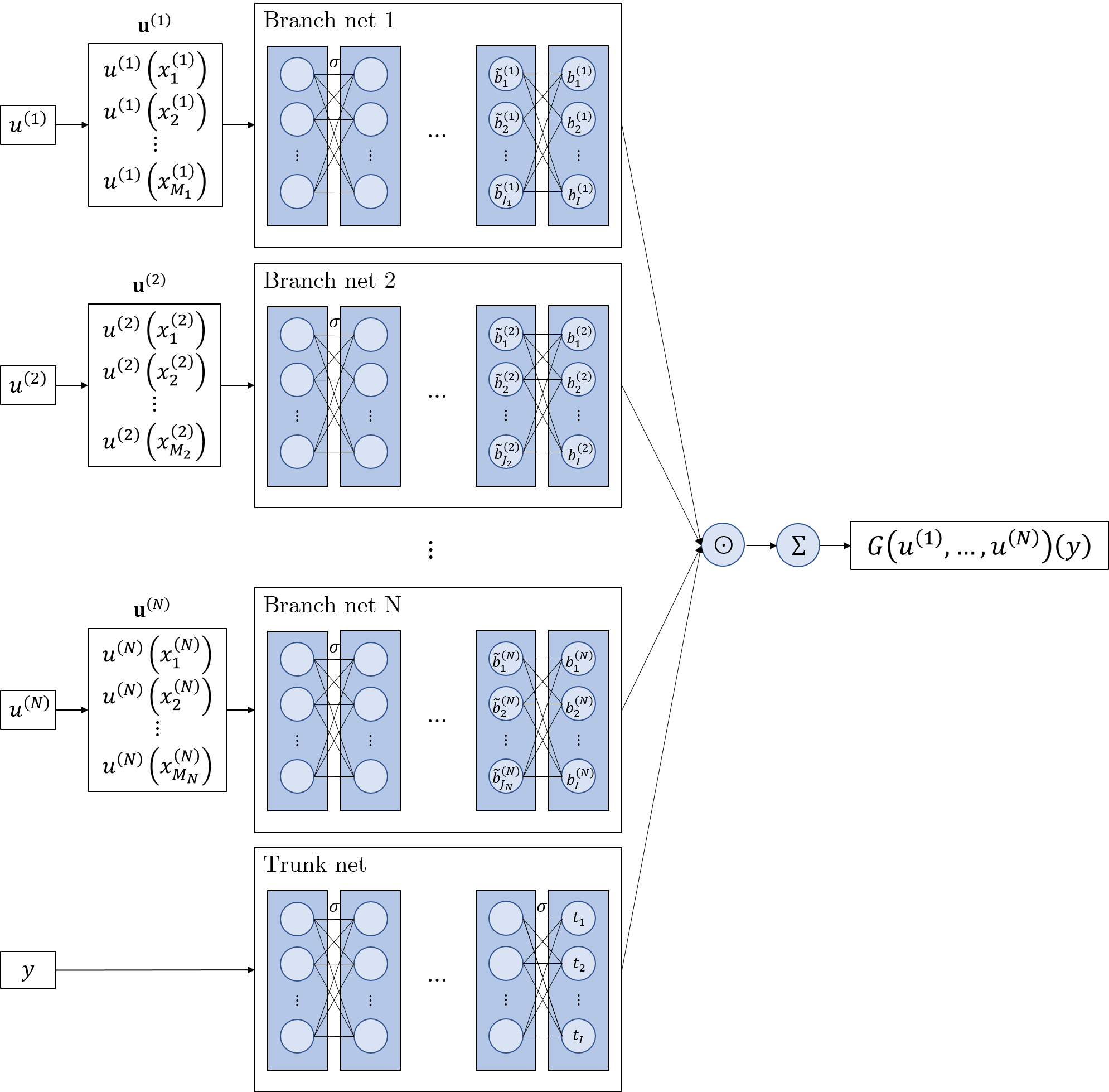}
    \caption{Structure of vanilla MIONet with fully connected layers.}\label{fig:21}
\end{figure}

We now address the UAT for vanilla MIONets, where the input 
functions are discretized function values. This theorem is a modification 
of~\cite[Theorem 2]{Hu2025} and~\cite[Theorem 2]{Lu2021}.

\begin{theorem}[Universal Approximation Theorem for vanilla MIONets]~\label{thm:UATVMION}
    Let $\tilde{X}_m$ be Banach spaces, $\tilde{K}_m\subset\tilde{X}_m$,
    $K_0\subset\mathbb{R}^{d_0}$ be compact subsets of $\tilde{X}_m$ and $\mathbb{R}^{d_0}$, 
    respectively, $V_m$ be a compact set in $C(\tilde{K}_m)$, and $G$ be a continuous operator 
    which maps $V_1\times\cdots\times V_N$ into $C(K_0)$. 
    Then, for any $\epsilon >0$, there exist positive integers $I,M_m$, continuous 
    vector functions $\mathbf{b}_m\in C(\mathbb{R}^{M_m},\mathbb{R}^{I})$ and 
    $\mathbf{t}\in C(\mathbb{R}^{d_0},\mathbb{R}^{I})$, and points 
    $x^{(m)}_{i} \in \tilde{K}_m$ with $m=1,\dots,N$ and $i=1,\dots,M_m$ such that
    \begin{equation}~\label{eq:UATVMION} 
        {\Biggl| G(u^{(1)},\dots,u^{(N)})(y) - 
        \Big\langle 
        \underbrace{\mathbf{b}_1(\mathbf{u}^{(1)})}_{\textrm{branch $1$}}
        \odot \cdots \odot 
        \underbrace{\mathbf{b}_N(\mathbf{u}^{(N)})}_{\textrm{branch $N$}}, 
        \underbrace{\mathbf{t}(y)}_{\textrm{trunk}}\Big\rangle\Biggr|} 
        < \epsilon
    \end{equation}
    holds for all $u^{(m)}\in V_m$ and $y\in K_0$,
    where $\langle \cdot,\cdot \rangle$ denotes the inner product in $\mathbb{R}^{I}$ and
    $\mathbf{u}^{(m)}={\left[u^{(m)}(x_{1}^{(m)}),\dots, u^{(m)}(x_{M_m}^{(m)})\right]}^{T}
    \in{\mathbb{R}}^{M_m}$ is a discretized input function $u^{(m)}$ at 
    ${\{{x_{i}^{(m)}}\}}_{i=1}^{M_m}$.
    Moreover, the functions $\mathbf{b}_m$ and $\mathbf{t}$ can be chosen as diverse classes 
    of neural networks satisfying the classical UAT of functions,
    where the branch networks$\mathbf{b}_m (\mathbf{u}^{(m)}) = C_m \circ 
    \tilde{\mathbf{b}}_m (\mathbf{u}^{(m)})$ are decomposed 
    into a hidden layer part $\tilde{\mathbf{b}}_m\colon \mathbb{R}^{M_m} \to \mathbb{R}^{J_m}$ 
    and a linear last layer part $C_m\colon \mathbb{R}^{J_m} \to \mathbb{R}^{I}$ with a matrix 
    parametrization $C_m = (c_{i j}^{(m)})\in \mathbb{R}^{I\times J_m}$.
\end{theorem}
The MIONet approximating $G$ in~\eqref{eq:UATVMION} can be expressed as a 
nested form of summations and products:
\begin{equation}\label{eq:MION}
    \begin{alignedat}{2} 
        G(u^{(1)},\dots,u^{(N)})(y) 
        &\approx \left\langle\mathbf{b}_{1}(\mathbf{u}^{(1)})\odot\cdots\odot\mathbf{b}_{N}
        (\mathbf{u}^{(N)}),\mathbf{t}(y)\right\rangle\\
        &= \sum_{i=1}^{I} \left(\prod_{m=1}^{N}b_{i}^{(m)}(\mathbf{u}^{(m)})\right) t_i(y)\\
        &= \sum_{i=1}^{I}\left[\prod_{m=1}^{N}\left(\sum_{j=1}^{J_m} c_{i j}^{(m)}
        \tilde{b}_{j}^{(m)}(\mathbf{u}^{(m)})\right)\right]t_i(y),
    \end{alignedat}
\end{equation}
where $\tilde{b}_{j}^{(m)}$ is the $j$-th component of the vector function
$\tilde{\mathbf{b}}_m$, and $b_i^{(m)}$ and $t_i$ are the $i$-th component of the vector 
function ${\mathbf{b}}_m$ and ${\mathbf{t}}_m$, respectively.

\subsection{Hybrid Least Squares/Gradient Descent Method for DeepONets}\label{sec22}

Cyr et al.~\cite{Cyr2020} suggested the hybrid least squares/gradient descent~(LSGD) method to 
improve deep neural network training. This method alternates between least squares~(LS) 
steps and gradient descent~(GD) steps. 
The LS steps solve an LS problem for the linear last layer parameters to optimize 
those parameters, and the GD steps optimize the hidden layer parameters with a gradient 
descent type optimizer. 

We have proposed the hybrid LSGD method for DeepONets in~\cite{Choi2025}, where we accelerate 
and improve training for vanilla DeepONets in both supervised and unsupervised learning. 
As a starting point for the LSGD method for MIONets, we briefly explain the LSGD method for 
DeepONets in this section. 

Consider a vanilla DeepONet, which is the vanilla MIONet in~\eqref{eq:MION} with $N=1$.
We use a general type of $l_2$ loss, which consists of the sum of squared $l_2$ error terms 
and a regularization term for the last layer parameters of the branch network: 
\begin{equation}~\label{eq:GenL2LSDON}
    \sum_{k=1}^{K} \epsilon_{k} {\left\|\mathcal{L}_k\left[G(\cdot)\right](u,y) -
    \mathcal{L}_k\left[\langle \mathbf{b}(\cdot;\theta^{B},\theta^{L}),
    \mathbf{t}(\cdot;\theta^{T})\rangle\right](\mathbf{u},y)\right\|}_{l_2(\chi_k)}^{2}
    + \lambda {\left\|\theta^{L}\right\|}_{2}^{2},
\end{equation}
where $\chi_k={\{(\hat{u}_{d_k},\hat{y}_{d_k})\}}_{d_k=1}^{D_k}$ is a finite collection of 
data pairs $(u,y)$, $\mathcal{L}_k$ is a linear operator between real-valued functions, 
$f_k=\mathcal{L}_k\left[G(\cdot)\right](\hat{u}_{d_k},\hat{y}_{d_k})\in\mathbb{R}^{D_k}$
is the given data for the $k$-th error term,
$\epsilon_k > 0$ and $\lambda\geq 0$ are the weights for each error term and the 
regularization term, and $\theta^{B}$, $\theta^{T}$, $\theta^{L}$ 
denote the parameters for the branch network except the last layer, the trunk network, 
and the last layer of the branch, respectively. 

We represent $\theta^{L}$ as the row-wise vectorization of the (branch) last layer 
parameter matrix $C\in\mathbb{R}^{I\times J}$:
\begin{equation*}~\label{eq:vecCDON}
    \theta^{L}=\text{vec}({C^{T}})=
    {[c_{11},\dots,c_{1J},c_{21},\dots,c_{2J},\dots,c_{I1},\dots,c_{IJ}]}^{T}
    \in{\mathbb{R}}^{IJ}.
\end{equation*}
Then,~\eqref{eq:GenL2LSDON} can be reformulated as an LS problem
\begin{equation}~\label{eq:LLLSDON} 
    \min_{\theta^{L}} {\sum_{k=1}^{K} \epsilon_{k} {\left\|f_k - A_k \theta^{L}\right\|}_{2}^{2}
    + \lambda {\left\|{\theta}^L\right\|}_{2}^{2}},
\end{equation}
where $A_k \in{\mathbb{R}}^{D_k\times IJ}$ is the matrix whose $(d_k,J(i-1)+j)$ entry is 
$\mathcal{L}_k\left[\tilde{b}_j t_i\right](\hat{\mathbf{u}}_{d_k},\hat{y}_{d_k})$.

Under additional conditions for the data $\chi_k$ and the linear operator $\mathcal{L}_k$ 
such that   
\begin{subequations}~\label{eq:CondDON} 
    \begin{align}
        \chi_k = \beta \times \tau_k &= {\left\{u_p\right\}}_{p=1}^{P} \times 
        {\left\{y_{q_k}\right\}}_{q_k=1}^{Q_k},\\
        \mathcal{L}_k\left[\tilde{b}_j t_i\right](\mathbf{u},y) 
        &= \tilde{b}_j(\mathbf{u}) \mathcal{L}_k\left[t_i\right](y),
    \end{align}
\end{subequations}
we have a main theorem for the LS step of the LSGD method for 
DeepONets~\cite[Theorem~3]{Choi2025}. 
In the theorem, the large matrix $A_k$ in \cref{eq:LLLSDON} can be factored into two small 
component matrices, each from the branch and trunk networks, using the Kronecker 
product, where a commutation matrix~\cite{MacRae1974,Magnus1979} is introduced to permute the 
row order of the product matrix.  

Then, the LS problem can be reformulated to a matrix equation, which is a special case of the 
generalized Sylvester equation of type $A X B + \lambda X = E$ with symmetric positive 
semi-definite matrices $A$ and $B$. Such matrix equation can be solved by using the spectral 
decompositions~\cite[Proposition 6]{Choi2025}. 
 
\begin{proposition}~\label{prop:SPSDSyl}
    Let $A\in\mathbb{R}^{R\times R}$ and $B\in\mathbb{R}^{S\times S}$ be symmetric positive
    semi-definite matrices,
    $E\in\mathbb{R}^{R\times S}$ be any matrix, and $\lambda$ be a nonnegative real number.
    Then, the solution of the matrix equation 
    \begin{equation}~\label{eq:SylEq}
        AXB+\lambda X = E
    \end{equation}
    is given as
    \begin{equation}~\label{eq:SylSol}
        X = Q_A \left[{(\mathbf{d}_A\mathbf{d}_B^{T}+\lambda \mathbf{1}_{R\times S})
        }^{\odot -1} \odot (Q_A^{T} E Q_B)\right] Q_B^{T},
    \end{equation}
    where $A = Q_A D_A Q_A^{T}$ and $B = Q_B D_B Q_B^{T}$ are the spectral
    decompositions with orthogonal matrices $Q_A$, $Q_B$ and diagonal
    matrices $D_A = \text{diag}(\mathbf{d}_A)$, $D_B = \text{diag}(\mathbf{d}_B)$
    when $\mathbf{d}_A\in\mathbb{R}^{R}$, $\mathbf{d}_B\in\mathbb{R}^{S}$. 
    Here, ${\odot^{-1}}$ denotes the entrywise inverse, and $\mathbf{1}_{R\times S}$ represents 
    the $R\times S$ matrix with all entries equal to one.
\end{proposition}

Therefore, the last layer parameters minimizing~\eqref{eq:LLLSDON} can be found 
efficiently, which is now the LS step of the LSGD method for DeepONets. 
The remaining hidden parameters can be found by using a GD-type optimizer. 
See~\cite[Algorithms 1 and 2]{Choi2025} for the workflow of the LSGD method for DeepONets and 
its practical modification with Adam optimizer~\cite{Kingma2017}, LS+Adam, respectively.

\section{Hybrid Least Squares/Gradient Descent Methods for MIONets}\label{sec3}

We consider the following loss of the vanilla MIONet, which consists of the sums of squared
$l_2$ error terms indexed by $k$ and regularization terms for the $N$ last layer parameters: 
\begin{equation}~\label{eq:GenL2LSMION}
    \begin{alignedat}{2}
        \sum_{k=1}^{K} \epsilon_{k} &\Big\|\mathcal{L}_k\left[G(\cdot,\dots,\cdot)\right]
        (u^{(1)},\dots,u^{(N)},y)\\
        &-\mathcal{L}_k\left[\langle\mathbf{b}_{1}(\cdot;\theta_{1}^{B},\theta_{1}^{L})
        \odot\cdots\odot\mathbf{b}_{N}(\cdot;\theta_{N}^{B},\theta_{N}^{L}),
        \mathbf{t}(\cdot;\theta^{T})\rangle\right](\mathbf{u}^{(1)},\dots,\mathbf{u}^{(N)},y)
        \Big\|_{l_2(\chi_k)}^{2}\\
        &+\sum_{m=1}^{N}\lambda_m {\left\|\theta_{m}^{L}\right\|}_{2}^{2},
    \end{alignedat}
\end{equation} 
where $\chi_k$ is a finite collection of data pairs 
${\left\{(\hat{u}_{d_k}^{(1)},\dots,\hat{u}_{d_k}^{(N)},\hat{y}_{d_k})\right\}}_{d_k=1}^{D_k}$, 
$\mathcal{L}_k$ is a linear operator between real-valued functions for the $k$-th error term,
$f_k$ is the data for the $k$-th error term given as
$f_k = \left({\mathcal{L}_k\left[G(\cdot,\dots,\cdot)\right](\hat{u}^{(1)}_{d_k},\dots,
\hat{u}^{(N)}_{d_k},\hat{y}_{d_k})}\right) \in{\mathbb{R}}^{D_k}$, 
$\epsilon_k > 0$ and $\lambda_m\geq 0$ are weights for error terms and the regularization 
terms, and $\theta_{m}^{B}$, $\theta^{T}$, $\theta_{m}^{L}$ 
denote the parameters for the $m$-th branch network except the last layer, the trunk network, 
and the last layer of the $m$-th branch, respectively, for $m=1,\dots,N$. 

For the convenience of notation, we use the big-endian lexicographic order 
defined below for the entries of the tensor. 
\begin{definition}\label{def:LexIndDef}
    Let $\mathcal{T}\in{\mathbb{R}}^{D_1\times\cdots\times D_N}$ be a rank $N$ tensor of size 
    $D_1\times\cdots\times D_N$. 
    The lexicographic order of the entry of $\mathcal{T}$ maps the coordinate index 
    $(d_1,\dots,d_N)$ to the corresponding big-endian lexicographic order 
    ${[d_1,\dots,d_N]}_{D_1,\dots,D_N}$ such that  
    \begin{equation}\label{eq:LexIndDef}
        {[d_1,\dots,d_N]}_{D_1,\dots,D_N} \coloneq \sum_{m=1}^{N-1} \left[ 
            \left(\prod_{l=m+1}^{N} D_{l}\right) (d_{m} - 1)
        \right] + d_N.
    \end{equation}
\end{definition}

Note that the above definition gives the row-wise lexicographic order of matrix 
entries. 

Let $C_{m}\in\mathbb{R}^{I\times J_m}$ be a parameter matrix of the $m$-th last layer, 
where $\theta_{m}^{L}$ is the row-wise vectorization of $C_{m}$. 
The entries of $C_{m}$ are sorted lexicographically in the order 
\cref{eq:LexIndDef}, i.e., 
\begin{equation*}~\label{eq:vecCMION}
    \theta_{m}^{L}=\text{vec}(C_{m}^{T})=
    {[c_{11}^{(m)},\dots,c_{1J_m}^{(m)},c_{21}^{(m)},\dots,c_{2J_m}^{(m)},
    \dots,c_{I1}^{(m)},\dots,c_{IJ_m}^{(m)}]}^{T}
    \in{\mathbb{R}}^{IJ_m}.
\end{equation*}

Here,~\eqref{eq:GenL2LSMION} 
is reformulated as the following minimization problem in terms of the $N$ last layer 
parameters $\{\theta_{1}^{L},\dots,\theta_{N}^{L}\}$: 
\begin{equation}\label{eq:LLALSMION}
    \min_{\{{\theta}_{1}^{L},\dots,{\theta}_{N}^{L}\}}
    \sum_{k=1}^{K} \epsilon_{k} {\left\|f_k - \mathcal{A}_k(\theta_{1}^{L},\dots,\theta_{N}^{L})
    \right\|}_{l_2(\chi_k)}^{2}
    + \sum_{m=1}^{N} \lambda_m {\left\|{\theta}_{m}^{L}\right\|}_{2}^{2},
\end{equation}
where
$\mathcal{A}_k(\theta_{1}^{L},\dots,\theta_{N}^{L}) \in {\mathbb{R}}^{D_k}$ is the long vector
whose $d_k$ entry is 
\begin{equation}\label{eq:EntryMION}
    \mathcal{L}_k\left[\sum_{i=1}^{I} b_{i}^{(1)} \cdots b_{i}^{(N)}t_i\right]
    (\hat{\mathbf{u}}_{d_k}^{(1)},\dots,\hat{\mathbf{u}}_{d_k}^{(N)},\hat{y}_{d_k}),
\end{equation}
and it is the result of an $N$-linear function applied to the last layer parameters
$\theta_{1}^{L},\dots,\theta_{N}^{L}$.

Although the minimization problem~\eqref{eq:LLALSMION} is not an LS problem, we can form an LS 
problem for $\theta_{n}^{L}$ by fixing all last layer parameters except the $n$-th parameter 
$\theta_n^L$. 
By solving the LS problems generated for each last layer parameter in turn, 
we can optimize the $N$ last layer parameters. The details will be discussed after formulating 
the LS problem for the single last layer parameter. 

Note that $\mathcal{A}_k(\theta_{1}^{L},\dots,\theta_{N}^{L})$ is linear with respect to 
$\theta_{n}^{L}$ for each $n$, we can write 
\begin{equation*}~\label{eq:MIONA}
    \mathcal{A}_k(\theta_{1}^{L},\dots,\theta_{N}^{L}) = A_{k,n} \theta_{n}^{L}
\end{equation*} 
for each $k$ and $n$.

Therefore, the LS step minimizing~\eqref{eq:LLALSMION} in terms
of $\theta_{n}^{L}$ is 
\begin{equation}~\label{eq:LLLSMION}
    \min_{{\theta}_{n}^{L}}
    \sum_{k=1}^{K} \epsilon_{k} {\left\|f_{k} - A_{k,n} \theta_{n}^{L}\right\|}_{2}^{2}
    + \lambda_n {\left\|{\theta}_{n}^{L}\right\|}_{2}^{2}
\end{equation}
where $A_{k,n}\in{\mathbb{R}}^{D_k\times IJ_n}$ is the matrix whose $(d_k,{[i,j]}_{I,J_n})$
entry is 
\begin{equation*}~\label{eq:MIONFactEntry}
    \mathcal{L}_k\left[b_{i}^{(1)}\cdots b_{i}^{(n-1)} \tilde{b}_{j}^{(n)} b_{i}^{(n+1)}
    \cdots b_{i}^{(N)}t_i\right](\hat{\mathbf{u}}_{d_k}^{(1)},\dots,\hat{\mathbf{u}}_{d_k}^{(N)},
    \hat{y}_{d_k}).
\end{equation*}

Here, the LS system~\eqref{eq:LLLSMION} cannot be handled 
directly if $I$ and $J_n$ are not small unless $D_k$ is small, but it is not desirable to 
use small $D_k$ in MIONet training. We provide the data and operator conditions for MIONet, 
which correspond to the conditions for DeepONet in~\cite[Eqs.~(3.4) and (3.5)]{Choi2025}.

Suppose that the data collection $\chi_k$ can be expressed as a Cartesian product of 
$N+1$ input data as follows: 
\begin{equation}~\label{eq:CondDataMION} 
    \chi_k = \beta_1 \times \cdots \times \beta_N \times \tau_k,
\end{equation} 
where $\beta_m = {\left\{{u}_{p}^{(m)}\right\}}_{p=1}^{P_m}$ is the set of input functions 
for the $m$-th coordinate and $\tau_k = {\left\{{y}_{q_k}\right\}}_{q_k=1}^{Q_k}$ 
is the set of points of the discretized domain for $\mathcal{L}_k [G(\cdot,\dots,\cdot)]$ 
with $D_k = P_1 \cdots P_N Q_k$. 
Thus, the input function for each branch network is independent of the input functions 
for other branch networks and is used equally for all error terms. 
Furthermore, the same discretization over the domain of 
$\mathcal{L}_k [G(\cdot,\dots,\cdot)]$ is used for all tuples of input functions 
$({u}_{p_1}^{(1)},\dots,{u}_{p_N}^{(N)})$ for each error term.

We also assume that the linear operator $\mathcal{L}_k$ satisfies 
\begin{equation}~\label{eq:CondOpMION} 
    \mathcal{L}_k\left[\tilde{b}_{j_1}^{(1)}\cdots\tilde{b}_{j_N}^{(N)} t_i\right]
    (\mathbf{u}^{(1)},\dots,\mathbf{u}^{(N)},y)
    = \left(\prod_{m=1}^{N}\tilde{b}_{j_m}^{(m)}(\mathbf{u}^{(m)})\right)
    \mathcal{L}_k\left[t_i\right](y)
\end{equation} 
for all $i,j_m,k$. 
This implies the linear operator $\mathcal{L}_k$ is independent of input functions and 
acts on each trunk component $t_i$.

Let $\ast$ denote the Khatri-Rao product (column-wise Kronecker product)~\cite{Khatri1968}. 
Note that the Khatri-Rao product of two matrices $X\in\mathbb{R}^{R_1\times S}$ and 
$Y\in\mathbb{R}^{R_2\times S}$ is an $R_1 R_2 \times S$ matrix 
whose entry is represented as 
\begin{equation}~\label{eq:KhaRaoNota}
    {(X \ast Y)}_{r s} = {X}_{r_1 s} {Y}_{r_2 s},
\end{equation} 
where $r = {[r_1,r_2]}_{R_1,R_2}$.

Under the conditions~\eqref{eq:CondDataMION} and~\eqref{eq:CondOpMION}, 
we show the large matrix $A_{k,n}$ in~\eqref{eq:LLLSMION} can be factored into the 
product of $N+1$ smaller matrices with some permutations. 
For convenience, we define a sequence with the $j$-th term removed from the sequence 
${\left\{x_i\right\}}_{i=1}^{N}$ and denote it as $(x_1,\dots,\hat{x}_j,\dots,x_N)$. 
Also, for the generic associative binary operator $\circ$, 
we denote $\overset{N}{\underset{\substack{i=1\\i\ne j}}{\circ}} x_i$ 
as the repeated application of 
the operation to the sequence ${\left\{x_i\right\}}_{i=1}^{N}$ with $x_j$ omitted.

We define the tensor permutation matrix introduced in~\cite{Rakotonirina2005} in terms of the 
lexicographic order of the entry of a rank $N$ tensor. 
\begin{definition}\label{def:TensorPermMat}
    Let $\mathcal{T}\in{\mathbb{R}}^{D_1\times\cdots\times D_N}$ be a rank $N$ tensor of size 
    $D_1\times\cdots\times D_N$ and 
    $\mathbf{t} \in \mathbb{R}^{D_1\cdots D_N}$ be the lexicographically vectorized tensor 
    $\mathcal{T}$ whose ${[d_1,\dots,d_N]}_{D_1,\dots,D_N}$ 
    entry is $\mathcal{T}_{d_1,\dots,d_N}$.
    Let $\sigma \in S_N$ be a permutation, where
    $S_N$ is the set of all permutations of $\{1,\dots,N\}$.
    Suppose $\tilde{\mathcal{T}}\in{\mathbb{R}}^{D_{\sigma(1)}\times\cdots\times D_{\sigma(N)}}$ 
    be a transposed tensor satisfying 
    $\mathcal{T}_{d_1,\dots,d_N} = \tilde{\mathcal{T}}_{d_{\sigma(1)},\dots,d_{\sigma(N)}}$
    with the lexicographic vectorization $\tilde{\mathbf{t}}$ for 
    $D_{\sigma(1)}\times\cdots\times D_{\sigma(N)}$ tensors. 
    If a permutation matrix ${K}_{D,\sigma}\in{\mathbb{R} 
    }^{D_1\cdots D_N \times D_1\cdots D_N}$ satisfies 
    \begin{equation}\label{eq:TensorPermMat}
        {\tilde{\mathbf{t}}} = {K}_{D,\sigma} {\mathbf{t}}, 
    \end{equation}
    then we call ${K}_{D,\sigma}$ a tensor permutation matrix with dimension 
    $D=(D_1,\dots,D_N)$ and permutation $\sigma$.
\end{definition}
Note that $\tilde{\mathcal{T}}$ is the tensor transpose of $\mathcal{T}$ associated with 
$\sigma$ in~\cite{Pan2014} and ${K}_{D,\sigma}$ maps the lexicographic order of the entry of 
$\mathcal{T}$ to that of the entry of $\tilde{\mathcal{T}}$. 

Now, we state and prove the main theorem for a factorization of the system matrix 
in~\cref{eq:LLLSMION}.
\begin{theorem}~\label{thm:MIONFact}
    Under the conditions~\eqref{eq:CondDataMION} and~\eqref{eq:CondOpMION}, 
    the system matrix 
    $A_{k,n}\in \mathbb{R}^{P_1 \cdots P_N Q_k \times IJ_n}$ in the least squares 
    problem~\eqref{eq:LLLSMION} can be factored as
    \begin{equation*}~\label{eq:MIONFact}
        A_{k,n} = {K}_{\tilde{D}_k,\sigma_n}^{T}\left(\left(G_{n}\ast T_k\right)
        \otimes B_n\right),
    \end{equation*}
    with
    \begin{equation*}~\label{eq:MIONFactDef}
        \begin{alignedat}{2}
            G_n &= \overset{N}{\underset{\substack{m=1\\m\ne n}}{\ast}} H_{m}
            \in {\mathbb{R}}^{P_1\cdots \hat{P}_{n}\cdots P_{N}\times I},\\
            H_m &= B_{m}C_m^{T}\in{\mathbb{R}}^{P_m\times I},
        \end{alignedat}
    \end{equation*}
    where 
    $B_m = \left({\tilde{b}_{j}^{(m)}({\mathbf{u}}_{p}^{(m)})}\right) \in 
    {\mathbb{R}}^{P_m \times J_n}$ and $T_k = \left({\mathcal{L}_k\left[t_i\right]({y}_{q_k})}
    \right) \in {\mathbb{R}}^{Q_k\times I}$ are the $m$-th branch pre-output matrix and 
    the trunk output data with operator $\mathcal{L}_k$, respectively, 
    and 
    ${K}_{\tilde{D}_k,\sigma_n}\in{\mathbb{R}}^{P_1\cdots P_N Q_k\times P_1\cdots P_N 
    Q_k}$ is the tensor permutation matrix with dimension $\tilde{D}_k=(P_1,\dots,P_N,Q_k)$ 
    and the permutation
    \[\sigma_n = \bigl(\begin{smallmatrix} 1 & \cdots & n-1 & n & n+1 & \cdots & N & N+1\\ 
    1 & \cdots & n-1 & n+1 & n+2 & \cdots & N+1 & n\end{smallmatrix}\bigr)\in S_{N+1}.\]
\end{theorem}
Note that $H_m$ is the final output of the $m$-th branch network, and 
since the Khatri-Rao product is associative, we generalize~\eqref{eq:KhaRaoNota} to 
\begin{equation}~\label{eq:KhaRaoNotaGen}
    {(X_1 \ast \cdots \ast X_N)}_{r s} =  {(X_1)}_{r_1 s} \cdots {(X_N)}_{r_N s},
\end{equation} 
where $r = {[r_1,\dots,r_N]}_{R_1,\dots,R_N}$ for the matrices 
$X_n \in\mathbb{R}^{R_n \times S}$. A similar result holds for the Kronecker product. 
The tensor permutation matrix ${K}_{\tilde{D}_k,\sigma_n}$ is introduced to 
match the mixed row order of the product $(G_n\ast T_k)\otimes B_n$ to the lexicographic 
row order of the data $f_k$.
\begin{proof}
    By~\eqref{eq:CondOpMION}, for each $d_k$, there exist 
    positive integers $p_1,\dots,p_N$ and $q$ such that
    $d_k = {[p_1,\dots,p_N,q]}_{P_1,\dots,P_N,Q_k}$, where $\hat{u}_{d_k}^{(m)} = {u}_{p_m
    }^{(m)} \in \beta_m$ and
    $\hat{y}_{d_k} = {y}_{q}\in\tau_k$. Also, the tuple $(p_1,\dots,p_N,q)$ and $d_k$ have
    one-to-one correspondence as $1\leq p_m \leq P_m$ for $m=1,\dots,N$, $1\leq q \leq Q_k$,
    and $1\leq d_k \leq D_k = P_1 \cdots P_N Q_k$.

    Hence, by~\eqref{eq:CondDataMION}, $({[p_1,\dots,p_N,q]}_{P_1,\dots,P_N,Q_k},
    {[i,j]}_{I,J_n})$ entry of $A_{k,n}$ can be expressed as
    \begin{equation}~\label{eq:MIONFactExpress}
        \left(\prod_{\substack{m=1\\m\ne n}}^{N} {b}_{i}^{(m)}({\mathbf{u}}_{p_m}^{(m)})
        \right) \tilde{b}_{j}^{(n)}({\mathbf{u}}_{p_n}^{(n)})
        \mathcal{L}_k\left[t_i\right]({y}_{q}).
    \end{equation}

    On the other hand, by using similar argument in~\eqref{eq:KhaRaoNotaGen}, 
    \[({[p_1,\dots,\hat{p}_{n},\dots,p_{N},q,p_n]}_{P_1,\dots,\hat{P}_{n},\dots,P_{N},Q_k,P_n},
    {[i,j]}_{I,J_n})\] entry of the mixed product
    $\left(G_{n}\ast T_k\right) \otimes B_n$ is exactly~\eqref{eq:MIONFactExpress}.
    Using ${K}_{\tilde{D}_k,\sigma_n}$ to permute
    the row order of $\left(G_{n}\ast T_k\right) \otimes B_n$,
    we have $A_{k,n} = {K}_{\tilde{D}_k,\sigma_n}^{T}\left(\left(G_{n}\ast T_k\right)
    \otimes B_n\right)$ as desired.
\end{proof}

Using the result of \cref{thm:MIONFact}, let us observe the LS problem~\eqref{eq:LLLSMION}
and rearrange the data vector $f_k\in\mathbb{R}^{P_1\cdots P_N Q_k}$ into a tensor
$F_k\in\mathbb{R}^{P_1\times\cdots\times P_N\times Q_k}$. 
Both the $(p_1,\dots,p_N,q_k)$ entry of $F_k$ and 
the ${[p_1,\dots,p_N,q_k]}_{P_1,\dots,P_N,Q_k}$ entry of $f_k$ 
are 
\begin{equation}~\label{eq:MIONDataEntry}
    \mathcal{L}_k\left[G(\cdot,\dots,\cdot)\right](u^{(1)}_{p_1},\dots,u^{(N)}_{p_N},y_{q_k}).
\end{equation}
Therefore, the LS problem with respect to the last layer parameter is
\begin{equation}~\label{eq:LLLSMIONFinal}
    \min_{{C}_{n}}
    \sum_{k=1}^{K} \epsilon_{k} {\left\|f_k - {K}_{\tilde{D}_k,\sigma_n}^{T}
    \left(\left(G_{n}\ast T_k\right) \otimes B_n\right) \text{vec}(C_{n}^{T})
    \right\|}_{2}^{2} +
    \lambda_n {\left\|\text{vec}(C_{n}^{T})\right\|}_{2}^{2}.
\end{equation}

To write the normal equation concisely, we use the fact~\cite{Rao1970} that 
${(X \ast Y)}^{T} (X \ast Y) = (X^{T}X) \odot (Y^{T}Y)$
for all matrices $X\in\mathbb{R}^{R_1\times S}$ and $Y\in\mathbb{R}^{R_2\times S}$.
Here, the normal equation of~\eqref{eq:LLLSMIONFinal} is given as
\begin{equation}~\label{eq:NormalMION}
    \begin{alignedat}{2}
        &\left[\left(\sum_{k=1}^{K} \epsilon_{k} \left(G_n^{T} G_n\right) \odot 
        \left(T_k^{T} T_k\right)\right)
        \otimes \left(B_n^{T} B_n\right)\right] \text{vec}(C_{n}^{T})
        + \lambda_n \text{vec}(C_{n}^{T})\\
        &=\sum_{k=1}^{K} \epsilon_{k} \left({\left(G_{n}\ast T_k\right)}^{T} \otimes
        B_n^{T}\right) {K}_{\tilde{D}_k,\sigma_n}^{T} f_k.
    \end{alignedat}
\end{equation}
Introducing a long data matrix $\hat{F}_{k,n}\in\mathbb{R}^{P_n \times P_1\cdots
\hat{P}_{n} \cdots P_N Q_k}$ whose 
\[\left(p_n,{[p_1,\dots,\hat{p}_n,\dots{p}_N,q]}_{P_1,\dots,\hat{P}_n,\dots,P_N,Q_k}\right)\]
entry is~\eqref{eq:MIONDataEntry}, we can write~\eqref{eq:NormalMION} as a matrix equation:
\begin{equation}~\label{eq:NormalMatMION}
    B_n^{T} B_n C_n^{T} \left(\sum_{k=1}^{K} \epsilon_{k}
    \left(G_n^{T} G_n\right) \odot
    \left(T_k^{T} T_k\right)\right) + \lambda_n C_n^{T}
    = B_n^{T} \left(\sum_{k=1}^{K} \epsilon_{k} \hat{F}_{k,n}
    \left(G_{n}\ast T_k\right)\right).
\end{equation}
Note that $(G_n^{T} G_n) \odot (T_k^{T} T_k)$ can be effectively calculated without 
forming huge matrices since 
\begin{equation*}~\label{eq:MIONMatReduce}
    (G_n^{T} G_n) \odot (T_k^{T} T_k) = 
    \left(\overset{N}{\underset{\substack{m=1\\m\ne n}}{\odot}} (H_{m}^{T} H_{m})\right)
    \odot (T_k^{T} T_k),
\end{equation*} 
where each component is an $I \times I$ matrix generated from each network. 

In the case of unsupervised learning where the given data $F_k$ depends only on 
a single input argument $u^{(m_k)}$ of the operator $G(u^{(1)},\dots,u^{(N)})$, 
we can reduce the rank $(N+1)$ data tensor $F_k$ 
to the reduced matrix $\tilde{F}_{m_k}\in\mathbb{R}^{P_{m_k}\times Q_k}$, where $m_k$ is the 
corresponding index of the input function coordinate. We provide a related theorem.

\begin{theorem}~\label{thm:MIONDatReduce}
    Suppose the data tensor $F_k\in\mathbb{R}^{P_1\times\cdots\times P_N\times Q_k}$ 
    in~\eqref{eq:MIONDataEntry} is given as the case of unsupervised learning, where the data varies only
    along the $m_k$-th axis and the last axis.
    Then, 
    \begin{equation*}~\label{eq:MIONDatReduce}
        \hat{F}_{k,n} \left(G_{n}\ast T_k\right) = \begin{cases}
        \mathbf{1}_{P_n\times 1}
        \left(
        \overset{N}{\underset{\substack{m=1\\m\ne n}}{\odot}} \tilde{\mathbf{h}}^{T}_{m}
        \right)\quad&
        \text{for $n\ne m_k$},\\
        \left[\mathbf{1}_{P_n\times 1} \left(
        \overset{N}{\underset{\substack{m=1\\m\ne n}}{\odot}} \tilde{\mathbf{h}}^{T}_{m}
        \right)\right] \odot \left(\tilde{F}_{n} T_k\right)\quad&
        \text{for $n= m_k$},
        \end{cases}
    \end{equation*}
    where $\tilde{\mathbf{h}}^{T}_{m}\in{\mathbb{R}}^{1\times I}$ is a row vector of 
    length $I$ defined by
    \begin{equation*}~\label{eq:MIONDatDef}
        \tilde{\mathbf{h}}^{T}_{m} =
        \begin{cases*}
            {\mathbf{1}_{1\times P_m}} H_{m}\quad&\text{if $m\ne m_k$},\\
            {\mathbf{1}_{1 \times P_{m_k}}} (H_{m_k} \odot (\tilde{F}_{m_k} T_k))
            \quad&\text{if $m = m_k$}.
        \end{cases*}
    \end{equation*}
\end{theorem}
\begin{proof}
    By expanding the summation in the matrix multiplication $\hat{F}_{k,n} \left(G_{n}\ast 
    T_k\right)$ of intermediate dimension $P_1\cdots \hat{P}_{n}\cdots P_{N} Q_k$ into $N$
    summations, we have
    \begin{equation}~\label{eq:MIONDatReduceSum}
        \begin{alignedat}{2} 
            {\left(\hat{F}_{k,n} \left(G_{n}\ast T_k\right)\right)}_{p_n i}
            &= \sum_{l} {(\hat{F}_{k,n})}_{p_n l} {(G_{n}\ast T_k)}_{l i}\\
            &= \sum_{p_1}\cdots\widehat{\sum_{p_{n}}}\cdots\sum_{p_N}\sum_{q}
            {(\tilde{F}_{m_k})}_{p_{m_k} q}
            \left(\prod_{\substack{m=1\\m\ne n}}^{N} {(H_{m})}_{p_m i}\right) {(T_k)}_{q i}.
        \end{alignedat} 
    \end{equation}
    If $n \ne m_k$, by rearranging summations with respect to their corresponding
    indices,~\eqref{eq:MIONDatReduceSum} becomes
    \begin{equation}\label{eq:MIONDatReduceSumNE}
        \begin{alignedat}{2}
            \Bigl(\hat{F}_{k,n} & \bigl(G_{n}\ast T_k\bigr)\Bigr)_{p_n i}\\ 
            &=\left(\prod_{\substack{m=1\\m\ne n,m_k}}^{N} \sum_{p_m} {(H_m)}_{p_m i}\right)
            \left(\sum_{p_{m_k}} \left({(H_{m_k})}_{p_{m_k} i} \sum_{q}
            {(\tilde{F}_{m_k})}_{p_{m_k} q} {(T_k)}_{q i}\right)\right)\\
            &=\left(\prod_{\substack{m=1\\m\ne n,m_k}}^{N} {\left({\mathbf{1}_{1 \times P_m}}
            H_{m}\right)}_{i}\right)
            \left(\sum_{p_{m_k}} {(H_{m_k})}_{p_{m_k} i}
            {(\tilde{F}_{m_k}T_k)}_{p_{m_k} i}\right)\\          
            &= \left(\prod_{\substack{m=1\\m\ne n,m_k}}^{N} {\left({\mathbf{1}_{1 \times P_m}}
            H_{m}\right)}_{i}\right) {\left({\mathbf{1}_{1 \times P_{m_k}}}
            (H_{m_k} \odot (\tilde{F}_{m_k} T_k))\right)}_{i}.
        \end{alignedat}
    \end{equation}
    Therefore, by reordering products and broadcasting along the first axis of the matrix
    as~\eqref{eq:MIONDatReduceSumNE} does not depend on the first axis variable $p_n$ we have
    \begin{equation*}~\label{eq:MIONDatReduce1}
        \hat{F}_{k,n} \left(G_{n}\ast T_k\right) = \mathbf{1}_{P_n\times 1}
        \left(\overset{N}{\underset{\substack{m=1\\m\ne n}}{\odot}} \tilde{\mathbf{h}}^{T}_{m}
        \right).
    \end{equation*}

    On the other hand, if $n = m_k$,~\eqref{eq:MIONDatReduceSum} becomes
    \begin{equation}~\label{eq:MIONDatReduceSumE}
        \begin{alignedat}{2}
            {\left(\hat{F}_{k,n} \left(G_{n}\ast T_k\right)\right)}_{p_n i} &= 
            \left(\prod_{\substack{m=1\\m\ne n}}^{N} \sum_{p_m} {(H_m)}_{p_m i}\right)
            \left(\sum_{q}{(\tilde{F}_{n})}_{p_{n} q} {(T_k)}_{q i}\right)\\
            &=\left(\prod_{\substack{m=1\\m\ne n}}^{N} {\left({\mathbf{1}_{1 \times P_m}} 
            H_{m} \right)}_{i}\right) {\left(\tilde{F}_{n} T_k\right)}_{p_n i}.
        \end{alignedat}
    \end{equation}
    Therefore, by reordering products and broadcasting terms along the first axis of the
    matrix which does not depend on the variable $p_n$ for~\eqref{eq:MIONDatReduceSumE}, we have
    \begin{equation*}~\label{eq:MIONDatReduce2}
        \hat{F}_{k,n} \left(G_{n}\ast T_k\right) = 
        \left[\mathbf{1}_{P_n\times 1} \left(\overset{N}{\underset{\substack{m=1\\m\ne n}}{\odot}}
        \tilde{\mathbf{h}}^{T}_{m}
        \right)\right] \odot \left(\tilde{F}_{n} T_k\right),
    \end{equation*}
    which concludes the proof. 
\end{proof}

This theorem enables us to efficiently compute each component of the matrix 
equation~\eqref{eq:NormalMatMION}. 
Finally, we can use \cref{prop:SPSDSyl} to find the minimizer $C_n$ of the LS 
problem~\eqref{eq:LLLSMIONFinal} for the $n$-th last layer parameter.

Now, we present the entire LS step for MIONet (\Cref{alg:ALSMION}). 
Since the loss is the sum of squares of multilinear 
functions with last layer parameters $(\theta_{1}^{L},\dots,\theta_{N}^{L})$ of $N$ branches, 
along with regularization terms, we can formulate an alternating least squares (ALS) problem. 
Define the optimizing order of the last layer parameters as a permutation 
$\pi\in S_N$. 
For $m=1,\dots,N$, fix all parameters except $\theta_{\pi(m)}^{L}$, then solve the LS problem 
in terms of $\theta_{\pi(m)}^{L}$ to update $\theta_{\pi(m)}^{L}$. Note that a unique 
minimizer is guaranteed if $\lambda_{\pi(m)}>0$. 
Finally, the entire process is repeated until the stopping criterion is 
met, such as changes in the appropriate norm of the parameters or the loss being 
smaller than a specific threshold. The complete LSGD algorithm for MIONet 
(\Cref{alg:LSGDMION}) is similar to the LSGD algorithm for DeepONet introduced 
in~\cite[Algorithm 2]{Choi2025}. 
To speed up the ALS step, one may introduce enhanced line search (ELS)~\cite{Rajih2008} 
or partitioned ALS (PALS)~\cite{Tichavsky2016}. 

\begin{algorithm}[tb]
    \caption{Alternating Least Squares Step for MIONet}\label{alg:ALSMION}
    \hspace*{\algorithmicindent} \textbf{Output: }
    Optimized last layer parameters $\theta_{1}^{L},\dots,\theta_{N}^{L}$
    \begin{algorithmic}[1]
        \Function{ALS}{$\theta_{1}^{B},\dots,\theta_{N}^{B},\theta^{T},
        \theta_{1}^{L},\dots,\theta_{N}^{L}$}
            \While{Stopping criterion is not met}
            \State{Choose permutation $\pi\in S_N$}
            \For{$m=1,\dots,N$}
                \State{$\theta_{\pi(m)}^{L} \gets
                LS_{\pi(m)}(\theta_{1}^{B},\dots,\theta_{N}^{B},\theta^{T},
                \theta_{1}^{L},\dots,\theta_{N}^{L})$}

                \(\triangleright\) Solve the LS problem in terms of $\theta_{\pi(m)}^{L}$
                and update
            \EndFor{}
            \EndWhile{}
        \EndFunction{}
    \end{algorithmic}
\end{algorithm}

\begin{algorithm}[tb]
    \caption{Hybrid Least Squares/Gradient Descent for MIONet}\label{alg:LSGDMION}
    \hspace*{\algorithmicindent} {$\Theta^{B} = (\theta_{1}^{B},\dots,\theta_{N}^{B})$, 
    $\Theta^{L} = (\theta_{1}^{L},\dots,\theta_{N}^{L})$}
    \begin{algorithmic}[1]
        \State{$(\Theta^{B},\theta^{T},\Theta^{L})\gets 
        (\Theta_{0}^{B},\theta_{0}^{T},\Theta_{0}^{L})$: 
        Initial parameters for the branches and the trunk}
        \State{$\Theta^{L} \gets ALS(\Theta^{B},\theta^{T},\Theta^{L})$}
        \Comment{Solve the ALS problem for each $\theta_{m}^{L}$ in the full batch}
        \For{$i=1,\dots$}
            \State{$(\Theta^{B},\theta^{T}) \gets
            GD(\Theta^{B},\theta^{T},\Theta^{L})$}
            \Comment{Use a gradient descent type optimizer to find
            $\Theta^{B}$ and $\theta^{T}$ }
            \State{$\Theta^{L} \gets ALS(\Theta^{B},\theta^{T},\Theta^{L})$}
        \EndFor{}
    \end{algorithmic}
\end{algorithm}

\section{Experimental results}\label{sec4}

In this section, we present numerical experiments on various PDEs with 2-branch MIONet 
to evaluate the proposed hybrid training schemes, ALS+Adam. 
We report training results and compare the proposed method,  
ALS+Adam, with the 
conventional Adam training in terms of training loss convergence behavior and 
relative error for unseen input functions. 

We propose ALS+Adam as a practical LSGD method for MIONets, which generalizes the LS+Adam 
method for DeepONets in~\cite{Choi2025} as follows. 
In the initial stage, we train all parameters using Adam for a moderate number of epochs. 
Then, we switch to the hybrid stage, where we use the ALS step to optimize the last layer 
parameters of each branch network in turn. 
After that, the ALS step is applied once every few Adam epochs for the 
hidden layer parameters. See \cref{alg:ALSAdam}. 

\begin{algorithm}[tb]
    \caption{ALS+Adam for MIONet}\label{alg:ALSAdam}
    \hspace*{\algorithmicindent} {$\Theta^{B} = (\theta_{1}^{B},\dots,\theta_{N}^{B})$, 
    $\Theta^{L} = (\theta_{1}^{L},\dots,\theta_{N}^{L})$}
    \begin{algorithmic}[1]
        \State{$(\Theta^{B},\theta^{T},\Theta^{L})\gets 
        (\Theta_{0}^{B},\theta_{0}^{T},\Theta_{0}^{L})$: Initial parameters}

        \For{$i=1,\dots,I_0$}
            \State{$(\Theta^{B},\theta^{T},\Theta^{L}) \gets
            Adam(\Theta^{B},\theta^{T},\Theta^{L})$}
            \Comment{Initial Adam stage for all parameters} 
        \EndFor{}

        \State{$\Theta^{L} \gets ALS(\Theta^{B},\theta^{T},\Theta^{L})$}
        \Comment{Solve the ALS problem for each $\theta_{m}^{L}$ in the full data batch}
        \For{$i=1,\dots$} \Comment{Work unit block}
            \For{$j=1,\dots,J_0$}
                \State{$(\Theta^{B},\theta^{T}) \gets Adam(\Theta^{B},\theta^{T},\Theta^{L})$} 
                \Comment{Use Adam for hidden layer parameters}
            \EndFor{}
            \State{$\Theta^{L} \gets ALS(\Theta^{B},\theta^{T},\Theta^{L})$}
        \EndFor{}
    \end{algorithmic}
\end{algorithm}

For 2-branch MIONet training in supervised learning, 
we use the mean square error (MSE) loss
\begin{equation}~\label{eq:Loss_Sup_MION}
    \epsilon_{1} L_{\text{data}} +
    \lambda_1 {\|C_1\|}_F^2 + \lambda_2 {\|C_2\|}_F^2,
\end{equation} and for the unsupervised learning, we use PI-loss
\begin{equation}~\label{eq:Loss_Unsup_MION}
    \epsilon_{1} L_{\text{data}} +
    \epsilon_{2} L_{\text{physics}} + 
    \lambda_1 {\|C_1\|}_F^2 + \lambda_2 {\|C_2\|}_F^2,
\end{equation} where ${\|\cdot\|}_F$ denotes the Frobenius norm, 
$\epsilon_{1} = 1$, 
$L_{\text{data}}$ is the $L^2$ MSE on the data pairs $(\hat{u}_{d_1},\hat{y}_{d_1})$
where $\hat{y}_{d_1}$ corresponds to the given data points of the governing PDE, 
and $L_{\text{physics}}$ is the $L^2$ MSE of the PDE residuals
$(\hat{u}_{d_2},\hat{y}_{d_2})$ where $\hat{y}_{d_2}$ corresponds to
the residual computation points.

Supervised learning requires labeled solutions for all possible pairs 
of input functions $(u^{(1)},u^{(2)})$.  
On the other hand, unsupervised learning does not require precomputed 
reference solutions during the training stage. Here, 
we ensure that each data tensor $F_k$ depends only on one input argument of the 
target operator. 
According to \cref{thm:MIONDatReduce}, we do not need 
to form large matrices, and efficient computation is possible in the ALS steps. 

In each experiment, we use hyperparameters and network settings similar to those used 
for DeepONets~\cite{Choi2025}. 
We use the Adam optimizer with $\text{lr} = 10^{-3}$ and $(\beta_1,\beta_2) = (0.99,0.999)$. 
The Adam momentums are maintained between the LS steps. 
He normal initialization~\cite{He2015} is used for parameter initialization, and 
the Swish function $x/(1+e^{-x})$ is used as the activation function. 
For training with Adam-only, no regularization terms for the last layer parameters are used, 
but for training with ALS+Adam, these regularization terms are applied with the same positive 
weights across the last layer parameters of each branch network, 
i.e., $\lambda_1 = \lambda_2 = \lambda$. 

We define one work unit (WU) as one Adam epoch for training all parameters 
and, in the hybrid stage, as a cycle of one Adam epoch followed by one ALS step 
in \cref{alg:ALSAdam}. 
For training with ALS+Adam, we assign 50 WUs for the initial stage and 
then switch to the hybrid stage after applying one ALS step. 
We empirically found that using only one ALS step for each WU is sufficient, as it greatly 
reduces computational time while maintaining convergence behavior. 
In each experiment, the last parameters of the first branch network (branch $f$) 
is optimized first in the ALS step.

For the training data, we generate 1,000 data instances independently for each of 
the two input functions to generate 1,000,000 pairs of input functions, which are all  
possible combinations. 
For validation data, we generate 4,000 data pairs. 
In the Adam-only training stage, $100 \times 100$ data block is used as one batch 
in each Adam iteration. 
On the other hand, training with ALS+Adam uses the batch size of $50 \times 50$
in Adam iterations. 
All experiments are repeated three times using different seeds for random 
initialization of parameters and batching for the Adam iterations. 

For each model structure and details for problem statement and data generation, 
refer to \cref{Tab:MION_config} and the corresponding subsections. 
All computations were performed using Google JAX~\cite{jax2018} on a machine 
with Intel Xeon Gold 6430 processors and NVIDIA GeForce RTX 4090 with 24 GB memory. 

\begin{table}
    \caption{\textit{Inputs for networks, network structures and 
    weight of MIONet models.} 
    IC and BC stand for initial condition and boundary condition, respectively.
    FCN and CNN stand for fully connected network and convolutional
    neural network, respectively. CNN structures are described in the corresponding
    subsections. Swish activation is used on all branches and trunks.}
    \centering 
    {\footnotesize
    \begin{tabular}{c c c c l c c}\toprule
        {Equation} & {Type} & {Network} & {Input} & 
        {{\begin{tabular}{c}Network structure\end{tabular}}} & 
        $\lambda$ & $\epsilon_2 $\\\midrule
        \multirow{6}{*}{{\begin{tabular}{c}Reaction-\\diffusion\end{tabular}}} & 
        \multirow{6}{*}{{\begin{tabular}{c}Supervised, \\Nonlinear\end{tabular}}} &
        {\begin{tabular}{c}Branch\\$f$\end{tabular}} & 
        {\begin{tabular}{c}1D\\Source\end{tabular}} & 
        {\begin{tabular}{l} FCN\\$[33, 150, 150, 150]$\end{tabular}} & 
        \multirow{6}{*}{$10^{-8}$} &
        \multirow{6}{*}{--}\\\cmidrule(lr){3-5} 
        & & {\begin{tabular}{c}Branch\\$D$\end{tabular}} & 
        {\begin{tabular}{c}1D\\Diffusivity\end{tabular}} & 
        {\begin{tabular}{l} FCN\\ $[33, 150, 150, 150]$\end{tabular}} & \\\cmidrule(lr){3-5} 
        & & Trunk & {$ (x,t)$} & 
        {\begin{tabular}{l} FCN\\ $[2, 150, 150, 150]$\end{tabular}} & & \\\midrule 
        \multirow{6}{*}{Advection} & 
        \multirow{6}{*}{{\begin{tabular}{c}Unsupervised, \\Linear\end{tabular}}} &
        {\begin{tabular}{c}Branch\\$f$\end{tabular}} & 
        {\begin{tabular}{c}1D\\Source\end{tabular}} & 
        {\begin{tabular}{l} FCN\\ $[33, 100, 100, 100]$\end{tabular}} & 
        \multirow{6}{*}{$10^{-6}$} &
        \multirow{6}{*}{$10^{-1}$}\\\cmidrule(lr){3-5} 
        & & {\begin{tabular}{c}Branch\\$PQ$\end{tabular}} & 
        BC+IC & 
        {\begin{tabular}{l} FCN\\ $[65, 100, 100, 100]$\end{tabular}} & \\\cmidrule(lr){3-5}
        & & Trunk & {$ (x,t)$} & 
        {\begin{tabular}{l} FCN\\ $[2, 100, 100, 100]$\end{tabular}} & & \\\midrule 
        \multirow{6}{*}{Poisson} & 
        \multirow{6}{*}{{\begin{tabular}{c}Unsupervised, \\Linear\end{tabular}}} &
        {\begin{tabular}{c}Branch\\$f$\end{tabular}} & 
        {\begin{tabular}{c}2D\\Source\end{tabular}} & 
        {\begin{tabular}{l} CNN + FCN\\$[1024, 150, 150]$\end{tabular}} & 
        \multirow{6}{*}{$10^{-12}$} &
        \multirow{6}{*}{$10^{-4}$}\\\cmidrule(lr){3-5} 
        & & {\begin{tabular}{c}Branch\\$g$\end{tabular}} & 
        BC & 
        {\begin{tabular}{l} FCN\\ $[129, 150, 150, 150]$\end{tabular}} & \\\cmidrule(lr){3-5} 
        & & Trunk & {$(x,y)$} & 
        {\begin{tabular}{l} FCN\\ $[2, 150, 150, 150]$\end{tabular}} & & \\\bottomrule 
    \end{tabular}~\label{Tab:MION_config}
    }
\end{table}

\subsection{Reaction-diffusion with variable source and diffusivity}\label{subsec41} 

In this section, we consider a reaction-diffusion equation with variable source and 
diffusivity as inputs: 
\begin{equation*}\label{eq:DiffReact}
    \begin{alignedat}{3} 
        \frac{\partial u}{\partial t} = \frac{\partial}{\partial x} \left(D(x)
        \frac{\partial u}{\partial x}\right) &+ R(u) + f(x),\quad
        &&(x,t)\in(0,1)\times(0,1],\\ 
        u(x,0) &= 0,\quad &&x\in(0,1),\\ 
        u(0,t) = u(1,t) &= 0,\quad &&t\in(0,1], 
    \end{alignedat} 
\end{equation*} 
where $R(u) = 0.01 u^2$. 
Here, we aim to learn a solution operator that maps the source $f$ and 
diffusivity $D$ to the solution $u$ via MIONet. 
The inputs are generated from a Gaussian process~(GP) in the interval $[0,1]$ 
with zero mean and a squared exponential covariance kernel  
\begin{equation}\label{eq:SqExpKer}
    k(x_1,x_2) = \sigma^2 \text{exp}\left(-\frac{{|x_1 - x_2|}^2}{2l^2}\right), 
\end{equation} 
with a pair of scale and variance parameters $(l,\sigma^2) = (0.2, 1)$ for $f$ 
and $(l,\sigma^2) = (0.2, 0.35)$ for $D$.
The input functions $f$ and $D$ are discretized at $33$ equidistant grid points of $[0,1]$, and 
the output function is evaluated at $33\times33$ equidistant grid points of 
${[0,1]}^2$. 
The reference solutions are generated by the finite difference method~(FDM)
using the Crank-Nicolson scheme and the first Taylor approximation for the reaction term.
The computational grid is size $129\times257$ in 
the space domain and time domain. 
For the training data, we note that 1,000,000 reference solutions from every pair of $(f,D)$ 
need to be prepared, which takes about 8 GB of memory. 

\subsection{Constant coefficient advection with variable source and IBC}\label{subsec42}

In this section, we solve a 1D constant coefficient advection equation with a source term 
via MIONet\@:
\begin{equation*}~\label{eq:Advec}
    \begin{alignedat}{3} 
        \frac{\partial u}{\partial t} + a\frac{\partial u}{\partial x} &= f(x),\quad
        &&(x,t)\in{(0,1]}^2,\\ 
        u(x,0) &= P(x),\quad &&x\in[0,1],\\ 
        u(0,t) &= Q(t),\quad &&t\in[0,1], 
    \end{alignedat}
\end{equation*}
where $a$ is a fixed constant and $P(0)=Q(0)$.
We aim to learn a solution operator which maps the source $f$, BC $P$, and IC $Q$ to the 
solution $u$.

Note that the analytical solution is given as follows: 
\begin{equation*}~\label{eq:ConstAdvecSol}
    u^{*}(x,t) = \begin{cases}
    \frac{1}{a}\left(F(x)-F(x-at)\right) + P(x-at), \quad &x-at \geq 0,\\
    \frac{1}{a}\left(F(x)-F(0)\right) + Q(t-\frac{x}{a}), \quad &x-at < 0,
\end{cases}
\end{equation*}
where $F(x) = \int_{0}^{x} f(s)\, ds$ is the antiderivative of $f$.
Since this solution may have non-differentiable cusps along the line $x-at=0$, 
it is challenging to generate such a solution by minimizing the PI-loss via 
automatic differentiation in unsupervised learning. 
Therefore, we introduce an additional condition for input functions 
--- $f(0)=0$ and $P'(0) = -\frac{1}{a}Q'(0)$ --- to ensure differentiability of the solution. 
Also, instead of having separate IC and BC as input functions, $P$ and $Q$, they can be 
naturally concatenated into one input function along the domains of IC and BC\@.
Let $\mathbf{p}=[P(0)\cdots P(1)]$ and $\mathbf{q}=[Q(0)\cdots Q(1)]$ be the discretizations of 
$P$ and $Q$ along their domains, $[0,1]\times\{0\}$ and $\{0\}\times[0,1]$, respectively. 
Since $P(0)=Q(0)$, we can concatenate $\mathbf{p}$ and $\mathbf{q}$ by 
$\mathbf{r}=[Q(1)\cdots Q(0)=P(0) \cdots P(1)]$, where $\mathbf{q}$ is flipped and one of the 
duplicate values $P(0)$ or $Q(0)$ is removed. 
$\mathbf{r}$ will be used for the second input argument of the operator.

To generate input functions $f$, $P$ and $Q$, we use the GP with zero mean and a squared 
exponential covariance kernel in \cref{eq:SqExpKer}.
First, we sample $g$ from the GP in the interval $[0,1]$ with scale factor $l=0.2$ and 
variance $\sigma^2 = 1$, and $h$ from the GP in the interval $[-a,1]$ with $l=0.2$ and 
$\sigma^2 = 1$. Then, we set 
\begin{equation*}~\label{eq:AdvecInputs}
    \begin{alignedat}{3} 
        f(x) &= g(x) - g(0),\quad &&x \in [0,1],\\
        P(x) &= h(x),\quad &&x \in [0,1],\\
        Q(t) &= h(-at),\quad &&t \in [0,1].
    \end{alignedat}
\end{equation*} 
The generated $f$, $P$ and $Q$ satisfy $f(0)=0$, $P(0)=Q(0)$ and $P'(0) = -\frac{1}{a}Q'(0)$. 
In this problem, we choose $a = 0.5$. 
The original input functions $f$, $P$ and $Q$ are discretized at $33$ equidistant grid points 
of $[0,1]$, and the concatenated input $\mathbf{r}$ is a vector of length $65$. 
The output function is evaluated on $33\times33$ equidistant grid points of ${[0,1]}^2$.
The antiderivative of the exact solution is computed by the composite trapezoidal rule in 
finer grids of size $129$ instead of $33$. 

\subsection{2D Poisson equation with variable source and Dirichlet BC}\label{subsec43}

In this section, we solve a 2D Poisson equation on the unit square with Dirichlet BC via 
MIONet\@:  
\begin{equation}~\label{eq:Poisson}
    \begin{alignedat}{3} 
        -\nabla\cdot\left(\kappa\nabla u\right) &= f ,\quad &&(x,y)\in\Omega={(0,1)}^2,\\ 
        u &= g,\quad &&(x,y)\in\partial\Omega, 
    \end{alignedat}
\end{equation} 
where $\kappa\equiv 1$ and the MIONet takes two inputs, source $f$ and Dirichlet BC $g$, and 
generates a solution $u$.
The reference solutions are generated by the finite difference method on finer grids of size 
$129 \times 129$.

The 2D input source $f$ is generated from a GP with zero mean and a 2D squared exponential 
covariance kernel
\begin{equation*}~\label{eq:SqExpKer2D}
    k(x_1,x_2,y_1,y_2) = \sigma^2 \text{exp}\left(-\frac{{|x_1 - x_2|}^2}{2{l_x}^2}
    -\frac{{|y_1 - y_2|}^2}{2{l_y}^2}\right),
\end{equation*} 
with scale factors $l_x=l_y=0.2$ and variance $\sigma^2 = 0.1$. 
We take $33\times33$ equidistant grid points in ${[0,1]}^2$ for the 
discretization of the input function and as the evaluation points of the output function. 
The CNN of the branch network consists of three layers with $3\times 3$, $2\times 2$, and 
$2\times 2$ kernels each with $2 \times 2$ strides. Since the channel sizes are $[1,16,32,64]$, 
the output of the CNN is a $4\times 4$ image of $64$ channels. 

For the BC input function $g$, we consider a flattened 1D function $\tilde{g}$ on the interval
$[0,4]$, generated by a GP with zero mean and periodic covariance kernel 
\begin{equation*}~\label{eq:PeriodicKer}
    k(x_1,x_2) = \sigma^2 \text{exp}\left(-\frac{2}{{l}^2}\sin^{2}
    {\left(\frac{\pi|x_1 - x_2|} {p}\right)}\right),
\end{equation*} 
where the scale factor, the period, and the variance are $l = 0.3$, $p = 4$, and 
$\sigma^2 = 0.1$, respectively, 
so that $g(\mathbf{h}(t)) = \tilde{g}(t)$ where 
$\mathbf{h}\colon[0,4]\to\partial\Omega$ is the arc length parametrization of $\partial\Omega$
such that 
\begin{equation*}~\label{eq:PoisFlatBC}
    \begin{alignedat}{3} 
        \mathbf{h}(t) &=
        \begin{cases}
            (t,0),&t\in[0,1),\\ 
            (1,t-1),&t\in[1,2),\\ 
            (3-t,1),&t\in[2,3),\\ 
            (0,4-t),&t\in[3,4]. 
        \end{cases}
    \end{alignedat}
\end{equation*} 
The input function $\tilde{g}$ is discretized on $129$ equidistant grid points in $[0,4]$, 
and the output function is evaluated on $33\times33$ equidistant grid points in 
${[0,1]}^2$. 

\subsection{Results}\label{subsec44}

\begin{figure}
    \centering
    \begin{subfigure}{\textwidth}
        \centering
        \includegraphics[width=\linewidth]{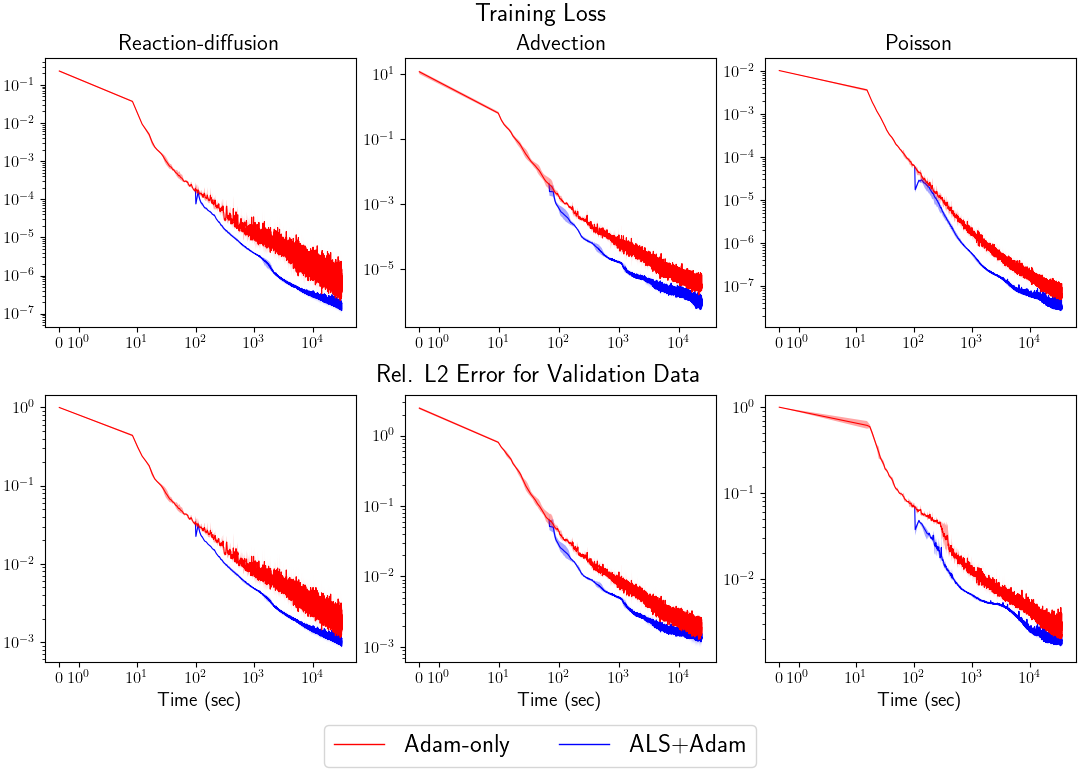}
    \end{subfigure}
    \caption{\textit{Solving various PDE problems via MIONet: 
    Adam-only (Red) vs. ALS+Adam (Blue).} 
    The top row denotes the training loss (without regularization terms in 
    \cref{eq:Loss_Sup_MION,eq:Loss_Unsup_MION}) for different seeds, 
    and the bottom row shows the mean relative $L^2$ error 
    for the validation data over time by the seeds. 
    The shaded areas represent the maximum and minimum values among the 
    seeds, and the solid lines represent their averages. 
    The plots are drawn on log-log scales.
    }\label{fig:Loss}
\end{figure} 

\begin{figure}
    \centering
    \begin{subfigure}{.95\textwidth}
        \centering
        \includegraphics[width=\linewidth]{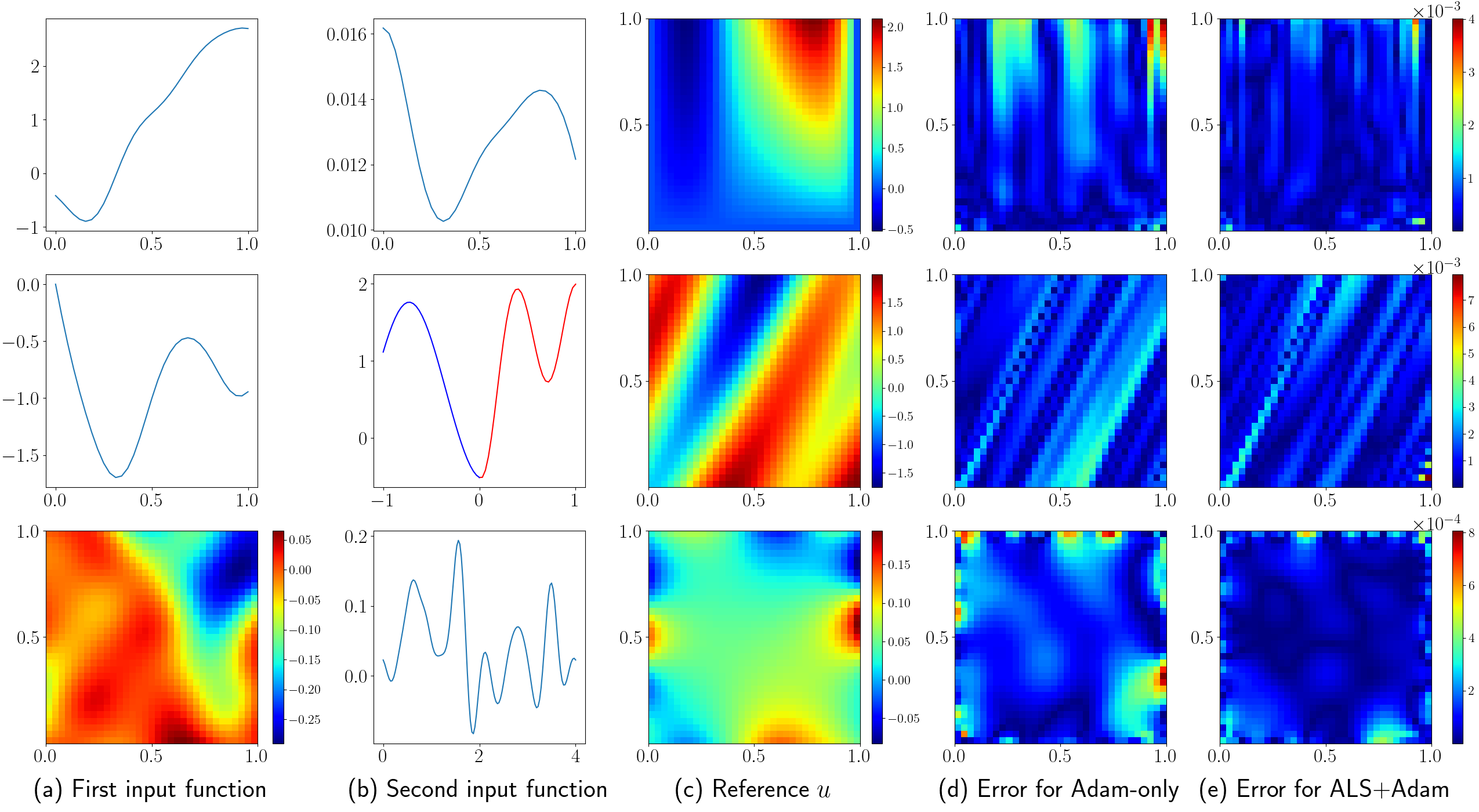}
    \end{subfigure}
    \caption{\textit{Test data evaluation for trained MIONet: 
    Adam-only vs. ALS+Adam.}
    From the top, the reaction-diffusion, advection, and Poisson 
    examples are illustrated. 
    The models are evaluated with parameters trained for  
    10,000 seconds. 
    }\label{fig:Test_result}
\end{figure}

As shown in \cref{fig:Loss}, training with ALS+Adam outperforms classical Adam training 
in terms of training loss decay and model performance (relative $L^2$ error for unseen data) 
in both supervised and unsupervised learning. 

\cref{fig:Test_result} illustrates the model errors of Adam-only's and ALS+Adam's for unseen 
pairs of test data functions at specific training time. Here, the errors for ALS+Adam training 
are significantly smaller than Adam-only's. 

We note that errors in both training results tend to form in specific patterns. 
For the reaction-diffusion example, the error grows as $t$ increases. In the advection example, 
the errors tend to form lines parallel to the line $x - 0.5 t = 0$. 
For the Poisson's equation, the errors are concentrated near the boundary. 

\section{Conclusion}\label{sec5}

In this paper, we propose a novel method to improve the training of vanilla MIONets, which 
generalizes the LSGD method for vanilla DeepONets~\cite{Choi2025}. 
By interpreting the general $L^2$ type of loss in terms of the last 
layer parameters of branch networks, we can view it as a sum of squared multilinear functions.
After that, we apply the ALS step 
to find the optimal set of last layer parameters that 
minimizes the loss, where we optimize by solving the corresponding LS problem of each last 
layer in turn.

Since each LS system is too large to handle directly, we factor the large matrix into 
small matrices corresponding to each branch and trunk network. 
Those small matrices constitute the original large matrix using the Kronecker and Khatri-Rao 
products together with an appropriate tensor permutation matrix. 
The solution of the LS system can be found in an elementary manner without forming large 
matrices. 
Finally, the LSGD method for MIONet alternates between the ALS step optimizing the last layer 
parameters and the GD step for the hidden layer parameters. 
The numerical experiments involving a nonlinear PDE with  
supervised learning and linear PDEs with PI-loss show that 
our method (ALS+Adam) accelerates training and generates better results with unseen 
data compared to the conventional Adam training. 

\bibliographystyle{siam}
\bibliography{references}

\end{document}